%% file: StructuralConnectomeAtlas.tex
\documentclass[runningheads]{llncs}
\usepackage{graphicx,subcaption}
\usepackage{amsmath,amssymb,amsfonts}
\usepackage{bbm}
\usepackage{todonotes}
\usepackage{float}
\usepackage{algorithm,algpseudocode}
\algnewcommand{\Inputs}[1]{%
  \State \textbf{Inputs:}
  \Statex \hspace*{\algorithmicindent}\parbox[t]{.8\linewidth}{\raggedright #1}
}
\algnewcommand{\Initialize}[1]{%
  \State \textbf{Initialize:}
  \Statex \hspace*{\algorithmicindent}\parbox[t]{.8\linewidth}{\raggedright #1}
}
\algnewcommand{\Outputs}[1]{%
  \State \textbf{Outputs:}
  \Statex \hspace*{\algorithmicindent}\parbox[t]{.8\linewidth}{\raggedright #1}
}
\include{MathNotation}

\graphicspath{{figures/}}

\begin{document}
\title{Structural Connectome Atlas Construction in the Space of Riemannian Metrics}
\titlerunning{Structural Connectome Atlas Construction}
\author{Kristen M. Campbell\thanks{M. Bauer was supported by NSF grants DMS-1912037, DMS-1953244. K. Campbell, H. Dai, S. Joshi and P. Fletcher were supported by NSF grant DMS-1912030. Z. Su was supported by NSF grant DMS-1912037, NIH/NIAAA award R01-AA026834. Data were provided in part by the Human Connectome Project, WU-Minn Consortium (Principal Investigators: David Van Essen and Kamil Ugurbil; 1U54MH091657) funded by the 16 NIH Institutes and Centers that support the NIH Blueprint for Neuroscience Research; and by the McDonnell Center for Systems Neuroscience at Washington University.}\inst{1}\and%\orcidID{0000-0001-8196-4423}\and
%\author{Kristen M. Campbell\inst{*1}\and%\orcidID{0000-0001-8196-4423}\and
Haocheng Dai\inst{1}\and
Zhe Su\inst{2}\and
Martin Bauer\inst{3}\and
P. Thomas Fletcher\inst{4}\and
Sarang C. Joshi\inst{1,5}}
\authorrunning{K. M. Campbell et al.}
\institute{Scientific Computing and Imaging Institute, University of Utah, Salt Lake City, UT \\
\email{kris@sci.utah.edu} \\
%\email{kris@sci.utah.edu, hdai@sci.utah.edu, sjoshi@sci.utah.edu} \\
\and Department of Neurology, University of California Los Angeles, Los Angeles, CA \\
%\email{zhesu@mednet.ucla.edu}
\and
Department of Mathematics, Florida State University, Tallahassee, FL\\
%\email{bauer@math.fsu.edu}
\and
Electrical \& Computer Engineering, University of Virginia, Charlottesville, VA \\
%\email{ptf8v@virginia.edu}
\and
Department of Bioengineering, University of Utah, Salt Lake City, UT \\
}

\maketitle              % typeset the header of the contribution

\begin{abstract}
The structural connectome is often represented by fiber bundles generated from various types of tractography. We propose a method of analyzing  connectomes by representing them as a Riemannian metric, thereby viewing them as points in an infinite-dimensional manifold. After equipping this space with a natural metric structure, the Ebin metric, we apply object-oriented statistical analysis to define an atlas as the Fr\'echet mean of a population of Riemannian metrics. We demonstrate connectome registration and atlas formation using connectomes derived from diffusion tensors estimated from a subset of subjects from the Human Connectome Project.

%  Existing methods cannot quantitatively compare tractograms between %subjects \todo{Because why?  And is saying cannot too strong?}.  The %connectome is \todo{always?} associated with a vector field, such as the %principal eigenvectors of a diffusion tensor field. Instead of working %with either fiber bundles or vector fields directly\todo{Difference %between and preferences for tractogram vs tractography vs fiber %bundles?}, we propose first estimating a Riemannian metric whose %geodesics coincide with the integral curves of the associated vector %field \todo{match singular vs plural}. We can then compute statistics on %the space of metrics including diffeomorphic registration, template %\todo{template vs atlas vs mean?} formation, and regression analysis to %compare the connectomes. While this method can be applied to a variety %of connectomes, in this paper we focus on connectomes derived from 2D %and 3D tensor fields.  We first demonstrate our method by computing the %Fr\'echet mean of synthetic connectomes.  We also demonstrate connectome %registration and template formation using connectomes derived from %diffusion tensors estimated for a subset of subjects from the Human %Connectome Project. \todo{More specifics about validation / verification here?}
%\keywords{structural connectome  \and diffeomorphic atlas \and Riemannian metric.}
\end{abstract}
\section{Introduction}
In this paper we develop for the first time statistical techniques on the  infinite-dimensional space of Riemannian metrics for analyzing the variability of the architecture of the human brain. Diffusion-weighted MRI (DWMRI) allows us to model an individual human brain as a Riemannian manifold with axonal connections that are geodesic curves of an appropriate metric. A Riemannian manifold is a topological manifold with an inner product defined on the tangent space at each point, the Riemannian metric. The Riemannian metric fundamentally defines the ``shape" of the manifold and defines the distance measured intrinsically on the manifold via geodesic curves. It is our fundamental assumption that the topology of the normal human brain is consistent across individuals, but the  difference in the connectomics  is because of the individual variation in the local Riemannian metric.

Several strategies have been used in previous work to construct white matter atlases from a population of diffusion MRI. Mori et al.~\cite{mori2008dti} construct a diffusion tensor imaging (DTI) atlas by registering the diffusion-weighted MRI of multiple subjects to a standardized anatomical template. They build the DTI atlas by transforming the diffusion tensors for each subject~\cite{alexander2001spatial} and then taking the Euclidean average of the transformed diffusion tensors at each voxel. This approach does not use the white matter directionality information encoded in the diffusion images during the registration. It also suffers from the fact that the Euclidean average of diffusion tensors does not take into account the directionality and tends to be fatter (i.e., less anisotropic) than the input tensors~\cite{fletcher2007riemannian}. Another approach by Yeh et al.~\cite{yeh2018population} is to register $q$-space diffusion images into an anatomical template and estimate the spin distribution function (SDF) at each voxel in the template. Then the SDFs are averaged on a per-voxel basis. While this method does take into account the directionality of the white matter in a local neighborhood, it does not take into account consistency of long-range white matter connections.

In this paper we develop a statistical groupwise atlas estimation algorithm for structural connectomes. The proposed algorithm uses not only local diffusion data but also long-range connectomics of the subjects as inferred by tractography~\cite{cheng2015tractography}. We do this by estimating a Riemannian metric of the brain manifold whose geodesic curves coincide with the tractography.

\section{Structural Connectomes as Riemannian Metrics}\label{sec:structasmet}
%Diffusion-weighted magnetic resonance imaging (DWMRI) measures the microscopic diffusion of water in multiple directions at every voxel in a 3D volume.  In the white matter of the brain, the diffusion of water is restricted perpendicular to the direction of the axons. Thus, the directionality of connections in the brain can be locally inferred.  
%Traditionally, global connections of the white matter have been estimated by a procedure called {\it tractography}, which numerically computes integral curves of the vector field formed by the most likely direction of fiber tracts at each point. Algorithms for estimating a Riemannian metric that captures connectomics of a subject have already been developed in diffusion imaging\ref{xxx,yyy}. In this work we build on the algorithm developed by~\cite{hao2014improved}, which assumes that the that the local Riemannian metric capturing the connectonomics modifies the inverse-tensor estimated from DTI by a spatially varying conformal factor. \todo{Tom and Kris Please add a description of the Algorithm and example results.}

 In the white matter of the brain, the diffusion of water is restricted perpendicular to the
 direction of the axons. Diffusion-weighted MRI measures the
 microscopic diffusion of water in multiple directions at every voxel in a 3D volume. Thus, the
 directionality of connections in the brain can be locally inferred. 
Traditionally, global connections of the white matter have been estimated by a procedure called {\it  tractography}, which numerically computes integral curves of the vector field formed by the most likely direction of fiber tracts at each point. DTI models connection directions with a tensor, $D(x)$, at each voxel whose principal eigenvector is aligned with the direction of the strongest diffusion.  

Riemannian metrics that represent connectomics of a subject have been developed in diffusion imaging~\cite{o2002new}  and include the inverse-tensor metric $\tilde g = D(x)^{-1}$. However, the geodesics associated with the inverse-tensor metric tend to deviate from the principal eigenvector directions and take straighter paths through areas of high curvature.

In this work we build on the algorithm developed by~\cite{hao2014improved}, which estimates a spatially-varying function, $\alpha(x)$, that modulates the inverse-tensor metric to create a locally-adaptive Riemannian metric, $g_\alpha = e^{\alpha(x)}\tilde g$.
We briefly describe the method here for completeness but refer the reader to~\cite{hao2014improved} for details. This adaptive \emph{connectome metric}, $g_\alpha$, is conformally equivalent to the inverse-tensor metric and is better at capturing the global connectomics, particularly through regions of high curvature.  Figure~\ref{fig:metricestgeos} shows how well the geodesics of each metric match the integral curve of the vector field.  The connectome metric geodesics are very closely aligned with the integral curves.

The geodesic between two end-points, $p, q$, associated with the inverse-tensor metric, $\tilde g(x) = D(x)^{-1}$, minimizes the energy functional, $\tilde E$. While the geodesic associated with the connectome metric, $g_\alpha(x) = e^{\alpha(x)}D(x)^{-1}$, minimizes the energy functional, $E_\alpha$:
\begin{equation}
   \begin{aligned}
   \tilde E(\gamma) = \int_0^1\langle T(t), T(t) \rangle_{\tilde g} dt,
   \end{aligned}  
   \qquad
   \begin{aligned}
   E_\alpha(\gamma) = \int_0^1e^{\alpha(x)}\langle T(t), T(t) \rangle_{\tilde g} dt,   
   \end{aligned}
\end{equation}
where $\gamma: [0, 1] \to M$, $\gamma(0) = p$, $\gamma(1) = q$, $T= \frac{d\gamma}{dt}$. 

Analyzing the variation of $E_\alpha$ leads to the geodesic equation, $\mathrm{grad}\, \alpha = 2 \nabla_T T$, where the Riemannian gradient of $\alpha$, $\mathrm{grad}\, \alpha = \tilde g^{-1}\bigl(\frac{\partial \alpha}{\partial x^1}, \frac{\partial \alpha}{\partial x^2}, \cdots, \frac{\partial \alpha}{\partial x^n}\bigr)$, and $\nabla_T T$ is the covariant derivative of $T$ along its integral curve.

%A vector field, $V$, can be written as $V = \sum \nu^i E_i$, where $E_i = \frac{\partial}{\partial x^i}$ are the coordinate basis vectors and $\nu^i$ are smooth coefficient functions.  The covariant derivative of a vector field, 
%\begin{equation}
%    \nabla_V V = \sum_k\biggl(\sum_i \nu^i \frac{\partial \nu^k}{\partial x^k} + \sum_{i,j} \Gamma_{ij}^k(\nu^i\nu^i)\biggr) E_k, 
%    \text{where } \Gamma_{ij}^k = \frac{1}{2}\sum_{l=1}^ng^{kl}\bigl(\frac{\partial g_{jl}}{\partial x^i} + \frac{\partial g_{il}}{\partial x^j} - \frac{\partial g_{ij}}{\partial x^l}\bigr),
%\end{equation}
% with Christoffel symbols, $\Gamma_{ij}^k$, measures how $V$ is changing along its integral curve.  

To enforce the desired condition where the tangent vectors, $T$, of the geodesic match the vector field, $V$, of the unit principal eigenvectors of $D(x)$, we minimize the functional, $F(\alpha) = \int_M || \mathrm{grad}\, \alpha - 2 \nabla_V V ||^2 dx$.  The equation for $\alpha$ that minimizes $F(\alpha)$ is 
\begin{equation}
\Delta \alpha = 2\, \mathrm{div} (\nabla_V V),
\label{eqn:poisson}
\end{equation}
where $\mathrm{div}$ and $\Delta$ are the Riemannian divergence and  Laplace-Beltrami operator.
% \begin{align}
% \mathrm{div}(X) = \frac{1}{\sqrt{|g|}}\sum_i\frac{\partial}{\partial x^i}(\sqrt{|g|}a^i) \\
%  \Delta \alpha = \mathrm{div} ( \mathrm{grad}\, \alpha ) = \frac{1}{\sqrt{|g|}}\sum_i\frac{\partial}{\partial x^i}(\sqrt{|g|}\sum_j g^{ij}\frac{\partial \alpha}{\partial x^j}).
% \end{align}
We discretize the Poisson equation in Equation \eqref{eqn:poisson} using a second-order finite difference scheme that satisfies both the Neumann boundary conditions $\frac{\partial \alpha}{\partial \overrightarrow{n}}= \langle \mathrm{grad}\, \alpha, \overrightarrow{n} \rangle = \langle 2\nabla_V V, \overrightarrow{n}\rangle$ and the governing equation on the boundary.  We then solve for $\alpha$.

Note that we can use this method to match the geodesics of the connectome metric to other vector fields defining the tractogram, e.g., from higher-order diffusion models that can represent multiple fiber crossings in a voxel. In particular, for tractography based on fiber orientation distributions (FODs), we can use the techniques presented in \cite{nie2019topographic} to generate the vector field $V$.

\begin{figure}[h]
\centering
\includegraphics[width=\linewidth]{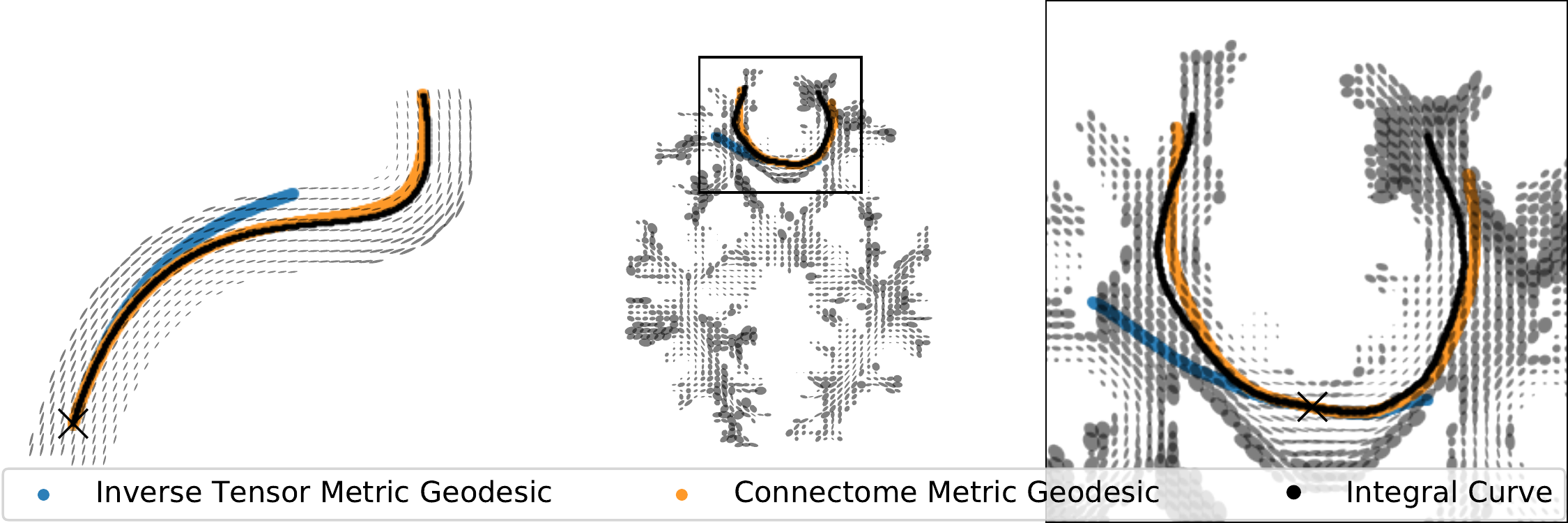}
\caption{A geodesic of the inverse-tensor metric (blue) and adaptive metric (orange), along with an integral curve (black) associated with the principal eigenvectors for a synthetic tensor field (left) and a subject's connectome metric from the Human Connectome Project (center).  Right shows a detailed view of the metric in the corpus callosum.}
\label{fig:metricestgeos}
\end{figure}
%\begin{figure}[h]
%\centering
%\includegraphics[width=\linewidth]{{cubic_and_brain_geodesics_with_integral_curve_filled}.pdf}
%\caption{Top: geodesics before (blue) and after (orange) matching the integral curves (black) associated with the principal eigenvectors for a synthetic tensor field (left) and a subject from the Human Connectome  (center).  Right shows a detailed view of the genu of the corpus callosum.}
%\label{fig:metricestgeos}
%\end{figure}

\section{The Geometry of the Manifold of all Metrics}\label{sec:riemet}
Once we have estimated a Riemannian metric for a human connectome, it is a point in the infinite-dimensional manifold, $\Met(M)$, where $M$ is the domain of the image. We will equip the infinite-dimensional space of all Riemannian metrics with a diffeomorphism-invariant Riemannian metric, called the Ebin or DeWitt metric~\cite{Ebin1970a,DeWitt67}. We base the statistical framework on this infinite-dimensional geometric structure. The invariance of the infinite-dimensional metric under the group of diffeomorphisms  $\Diff(M)$ is a crucial property, as it guarantees the independence of an initial choice of coordinate system on the brain manifold. In the following we will describe the details of our mathematical framework.

Let  $M$ be a smooth $n$-dimensional manifold; for our targeted applications $n$ will be two or three. We denote by $\Met(M)$ the space of all smooth Riemannian metrics on $M$, i.e., each element $g$ of the space $\Met(M)$ is a symmetric, positive-definite $0 \choose 2$ tensor field on $M$.  It is convenient to think of the elements of $M$ as being point-wise positive-definite sections of the bundle of symmetric two-tensors $S^2 T^\ast M$, i.e., smooth maps from $M$ with values in $S^2_+ T^\ast M$.  
Thus, the space $\Met(M)$  is an open subset of the linear space $\Ga(S^2 T^\ast M)$ of all smooth symmetric $0 \choose 2$ tensor fields and hence itself a smooth Fr\'echet-manifold \cite{Ebin1970a}. Furthermore, let $\Diff(M)$ denote the infinite-dimensional Lie group of all smooth diffeomorphisms of the manifold $M$. Elements of 
$\Diff(M)$ act as coordinate changes on the manifold $M$. This group acts on the space of metrics via pullback
\begin{align}\label{eq:diff_action}
\Met(M)\times\Diff(M)\to \Met(M), \qquad (g,\varphi)\mapsto \varphi^*g=g(T\varphi \cdot, T\varphi \cdot)\;.
\end{align}
It is important to note that the geometries of the metrics $g$ and $\varphi^*g$ are also related via $\varphi$. In particular, geodesics with respect to $g$ are mapped via $\varphi$ to geodesics with respect to $\varphi^*g$. %This observation will be essential for the our matching algorithm of the tracejo

On the infinite-dimensional manifold $\Met(M)$, there exists a natural Riemannian metric:
the reparameterization-invariant $L^2$-metric. To define the metric, we need to first characterize the tangent space of the manifold of all metrics:  $\Met(M)$ is an open subset of  $\Ga(S^2 T^\ast M)$. Thus, every  
tangent vector $h$ is a smooth bilinear form $h: TM \x_M TM \to \R$ that can be equivalently interpreted as a map $TM \to T^\ast M$.  The $L^2$-metric is given by
\begin{align}\label{eq.Ebinmetric}
    G^{E}_g(h,k)=\int_M \on{Tr}\big(g^{-1} hg^{-1} k\big)\vol(g),
\end{align}
with $g \in \Met(M)$, $h,k\in T_g\Met(M)$ and $\vol(g)$ the induced volume density of the metric $g$.  This metric, introduced in \cite{Ebin1970a}, is also known as the Ebin metric. We call the metric \emph{natural} as it requires no additional background structure and is consequently invariant under the action of the diffeomorphism group, i.e.,
\begin{equation}\label{eq:invariant}
G_g(h,k)=G_{\varphi^*g}(\varphi^*h,\varphi^*k)
\end{equation}
for all $\varphi \in \Diff(M)$, $g \in \Met(M)$ and $h,k\in T_g\Met(M)$. 
Note that the invariance of the metric follows directly from the substitution formula for multi-dimensional integrals.

The Ebin metric induces a particularly simple geometry on the space $\Met(M)$, with explicit formulas for geodesics, geodesic distance and curvature. In the following we will present the most important of these formulas, which will be of importance for our proposed metric matching framework.

First we note that a metric $g\in\Met(M)$, in local coordinates, can be represented as a field of symmetric, positive-definite $n\times n$ matrices that vary smoothly over $M$. Similarly, each tangent vector at $g$ can be represented as a field of symmetric $n\times n$ matrices. By the results of \cite{freed1989basic,gil1991riemannian,clarke2013geodesics}, one can reduce the investigations of the space of all Riemannian metrics to the study of the geometry of the finite-dimensional space of symmetric, positive-definite $n \times n$ matrices: the point wise nature of the Ebin metric allows one to solve the geodesic initial and boundary value problem on $\Met(M)$ for each $x\in M$ separately and thus the formulas for geodesics, geodesic distance and curvature on the finite-dimensional matrix space can be translated directly to results for the Ebin metric on the infinite-dimensional space of Riemannian metrics. %See \cite{freed1989basic,gil1991riemannian,clarke2013geodesics} for more details.

 Note that the space of Riemannian metrics, $\Met(M)$ with the Ebin metric, is not metrically complete and not geodesically convex. Thus the minimal geodesic between two Riemannian metrics may not exist in $\Met(M)$, but only in a larger space; the metric completion $\overline{
 \Met}(M)$, which consists of all possibly degenerate Riemannian metrics. This construction has been worked out in detail by Clarke~\cite{clarke2013completion} -- including the existence of minimizing paths in $\overline{\Met}(M)$. In the following we will omit these details and refer the interested reader to the article~\cite{clarke2013completion}  for a more in-depth discussion. In the following theorem, we present an explicit formula for the minimizing geodesic in $\overline{\Met}(M)$ that connects two given Riemannian metrics.
\begin{theorem}\label{thm:ebingeodesics}
	For $g_0, g_1\in\Met(M)$ we define
    \begin{align}
        k(x) &= \log\left(g_0^{-1}(x)g_1(x)\right),\quad k_0(x) = k(x) - \frac{\Tr(k(x))}{n}\Id\\
		a(x) &= \sqrt[4]{\det(g_0(x))},\quad b(x) = \sqrt[4]{\det(g_1(x))},\quad\ka(x) = \frac{\sqrt{n\Tr(k_0(x)^2)}}{4}\\
		q(t,x) &= 1+ t\left(\frac{b(x)\cos(\ka(x))-a(x)}{a(x)}\right),\quad	r(t,x) = \frac{t b(x)\sin(\ka(x))}{a(x)},
	\end{align}
	Then the minimal path $g(t,x)$ with respect to the Ebin metric in $\overline{\Met}(M)$ that connects $g_0$ to $g_1$ is given by
	\begin{align}
	    g = \begin{cases}
	        \left(q^2+r^2\right)^{\frac2n}g_0\exp\left(\frac{\arctan(r/q)}{\ka}k_0\right) & 0<\ka<\pi,\\
	        q^{\frac4n}g_0 & \ka=0,\\
	        \left(1-\frac{a+b}{a}t\right)^{\frac4n}g_0\mathbbm{1}_{\left[0,\frac{a}{a+b}\right]}+\left(\frac{a+b}{b}t-\frac{a}{b}\right)^{\frac4n}g_1 \mathbbm{1}_{\left[\frac{a}{a+b},1\right]} & \ka\geq\pi,
	    \end{cases}
	\end{align}
		where $\mathbbm{1}$ denotes the indicator function in the variable $t$. We suppressed the functions' dependence on $t$ and $x$ for better readability.
	
%	\begin{align}
%	g = \begin{cases}
%	\left(q^2+r^2\right)^{\frac2n}g_0\exp\left(\frac{4\arctan(r/q)k_0}{\sqrt{n\Tr(k_0^2)}}\right), &\ka_2<\frac{(4\pi)^2}{n}\text{ and } k_0\neq0,\\
%	q^{\frac4n}g_0, &\ka_2<\frac{(4\pi)^2}{n}\text{ and } k_0=0,\\
%	(1-2t)g_0 \mathbbm{1}_{[0,1/2)}+(2t-1)g_1 \mathbbm{1}_{[1/2,1]}
%,&\ka_2\geq \frac{(4\pi)^2}{n},
%	\end{cases}
%	\end{align}

%	if we have $\Tr(k_0(x)^2)<\frac{(4\pi)^2}{n}$; Or the concatenation of the straight segments from $g_0(x)$ to $0$ and from $0$ to $g_1(x)$, if we have $\Tr(k_0(x)^2)\geq\frac{(4\pi)^2}{n}$.
\end{theorem}
\begin{proof}
    This theorem is essentially a reformulation of the minimal geodesic formula given in \cite[Theorem 4.16]{clarke2013geodesics}. We obtain it by combining formulas for the exponential mapping, inverse exponential mapping, and minimal geodesic in \cite[Theorem 4.4, 4.5, 4.16]{clarke2013geodesics}. As these calculations are rather tedious we refrain from presenting them.
\end{proof}
We now recall that the geodesic distance of a Riemannian metric is defined as the infimum of all paths connecting two given points,
\begin{equation}
\dist_{\Met}(g_0,g_1)=\inf \int_0^1 \sqrt{G_g(\partial_t g,\partial_t g)} dt,
\end{equation} 
where the infimum is taken over all paths $g:[0,1]\to \Met(M)$ with 
$g(0)=g_0$ and $g(1)=g_1$. As a direct consequence of Theorem~\ref{thm:ebingeodesics} we obtain an explicit formula for this distance function:
\begin{corollary}
	Let $g_0, g_1\in\Met(M)$ and let $k, k_0$, $a$, $b$ and $\ka$ be as in Theorem~\ref{thm:ebingeodesics}. Let 
$%a(x) = \sqrt[4]{\det(g_0(x))},\; b(x) = \sqrt[4]{\det(g_1(x))},\; 
	\theta(x) = \min\left\{\pi, \ka(x)\right\}.
$
	Then the squared geodesic distance of the Ebin metric is given by:
	\begin{align}\label{eq:distance_function}
	\dist_{\Met}(g_0, g_1)^2 = \frac{16}{n}\int_M \left(a(x)^2 - 2a(x)b(x)\cos\left(\theta(x)\right) + b(x)^2\right)dx.
	\end{align}
\end{corollary}

Having equipped the space of Riemannian metric with the distance function~\eqref{eq:distance_function}, we can consider the Fr\'echet mean $\hat g$ of a collection of metrics $g_1,\ldots g_N$, which is defined as a minimizer of the sum of squared distances:
\begin{equation}
\hat{g}=\underset{g }{\operatorname{\rm argmin}}
    \sum_{i=1}^{N}\operatorname{dist}_{\Met}^{2}(g,g_i). 
\end{equation}
One could directly minimize this functional using a gradient-based optimization procedure. As our distance function is the geodesic distance function of a Riemannian metric and since we have access to an explicit formula for the minimizing geodesics, we will instead use the iterative geodesic marching algorithm, see e.g.~\cite{ho2013recursive}, to approximate the Fr\'echet mean. Given $N$ Riemannian metrics $g_i$, we approximate the Fr\'echet mean via $\hat g=\hat g_{N}$, where $\hat g_{i}$ is recursively defined as 
 $\hat g_0=g_0$, $\hat g_i(x)= g(1/(i+1),x)$ and where
 $g(t,x)$ is the minimal path, as given in Theorem~\ref{thm:ebingeodesics}, connecting $\hat g_{i-1}$ to the $i$-th data point $g_i$.
Thus one only has to calculate $N$ geodesics \emph{in total} in the space of Riemannian metrics, whereas a gradient-based algorithm would require one to calculate $N$ geodesic distances \emph{in each step} of the gradient descent.
%%-----------------------

\subsection{The induced distance function on the diffeomorphism group}\label{section:diffmetric}
We can use the geodesic distance function of the Ebin metric to induce a right-invariant distance function on the group of diffeomorphisms. As we will be using this distance function as a regularization term in our matching functional, we will briefly describe this construction here. We fix a Riemannian metric $g\in \Met(M)$ and define the ``distance'' of a diffeomorphism $\varphi$ to the identity via
\begin{equation}
    \operatorname{dist}_{\Diff}^{2}(\operatorname{id},\varphi) = \operatorname{dist}_{\Met}^{2}(g,\varphi^*g).
\end{equation}
To be more precise, this distance can be degenerate on the full diffeomorphism group since the isometries of the Riemannian metric $g$ form the kernel of $\operatorname{dist}_{\Diff}$. For our purposes we will consider the Euclidean metric for the definition of $\operatorname{dist}_{\Diff}$. Thus the only elements in the kernel are translations and rotations. The right invariance of $\operatorname{dist}_{\Diff}$ follows directly from the $\Diff(M)$-invariance of the Ebin metric. We note, however, that $\dist_{\Diff}$ is not directly associated with a Riemanian structure on the diffeomorphism group: the orbits of the diffeomorphism group in the space of metrics are not totally geodesic and thus $\dist_{\Diff}$ is not the geodesic distance of the pullback of the Ebin metric to the space of diffeomorphisms. See also~\cite{KLMP2013} where this construction has been studied in more detail.

\section{Computational Anatomy of the Human Connectome}
Fundamental to the precise characterization and comparison of the human connectome of an individual subject or a population as a whole is the ability to map or register two different human connectomes. The framework of Large Deformation Diffeomorphic Metric Mapping (LDDMM) is well developed for registering points~\cite{joshi2000landmark} curves~\cite{glaunes2008large} and surfaces~\cite{vaillant2005surface} all modeled as sub-manifolds of $\R^{3}$ as well as images modeled as an $L^2$ function~\cite{beg2005computing}.%,du2011whole
This framework has also been extended to densities~\cite{bauer2015diffeomorphic} modeled as volume forms. We now extend the diffeomorphic mapping framework to the  connectome modeled as Riemannian metrics. The diffeomorphisms group acts naturally on the space of metrics, see Equation \eqref{eq:diff_action}. With this action and a reparameterization-invariant metric, the problem of registering two connectomes  fits naturally into the framework of computational anatomy. We register two connectomes by solving the following minimization problem:
\begin{equation}
E(\varphi)=\underset{\varphi \in \Diff(M)}{\operatorname{inf}} \operatorname{dist}_{\Diff}^{2}(\operatorname{id},\varphi)+\lambda \operatorname{dist}_{\Met}^{2}(g_0,\varphi^*g_1)\label{eq:energy}
\end{equation}
where  $\operatorname{dist}_{\Diff}$ is a right invariant distance on $\Diff$ and $\operatorname{dist}_{\Met}$ is a reparameterization-invariant distance  on the space of all Riemannian metrics, e.g., the geodesic distance of the metrics studied above.  The first term measures the deformation cost and the second term is a similarity measure between the target and the deformed source connectome.  The invariance of the two distances is essential for the minimization problem to be independent of the choice of coordinate system on the brain manifold.

%To maximize computational efficiency we use the distance function as introduced in Section~\ref{section:diffmetric} to measure the deformation cost. Consequently, using the formulas from Section~\ref{sec:riemet}, we have explicit formulas for both terms in the energy functional. 
We use the distance function as introduced in Section~\ref{section:diffmetric} to measure the deformation cost, i.e., $\operatorname{dist}_{\Diff}(\operatorname{id},\varphi)=\operatorname{dist}_{\Met}(g,\varphi^*g)$ where $g$ is the restriction of the euclidean metric to the brain domain. 
This choice greatly increases computational efficiency since we can now use the formulas from Section~\ref{sec:riemet} as explicit formulas for both terms of the energy functional.
To minimize the energy functional, we use a gradient flow approach described in Algorithm~\ref{algo1}, where the gradient on $\Diff(M)$ is calculated with respect to a right invariant Sobolev metric of order one, called the information metric~\cite{bauer2015diffeomorphic}. We choose this specific gradient because of the relation of the information metric to both the Ebin metric on the space of metrics and the Fisher-Rao metric on the space of probability densities. See~\cite{KLMP2013,bauer2015diffeomorphic} for a precise description of the underlying geometric picture. 

% We use the distance function as introduced in Section~\ref{section:diffmetric} to measure the deformation cost.  This choice maximizes computational efficiency since we can now use the formulas from Section~\ref{sec:riemet} as explicit formulas for both terms of the energy functional.
% We use formulas from Section~\ref{sec:riemet} for this distance function to provide explicit formulas for both terms of the energy functional.
% This choice lets us the formulas from Section~\ref{sec:riemet} so that we have explicit formulas for both terms in the energy functional. 
% Consequently, using the formulas from Section~\ref{sec:riemet}, we have explicit formulas for both terms in the energy functional.

Note, that our framework allows for the immediate inclusion of  points, curves, surfaces and images in the registration problem, which we plan to incorporate in future work. Image intensity information, for example, can be easily incorporated in the registration problem by simply adding an appropriate similarity measure for the image term (e.g. the standard $L^2$ metric between the deformed source image and the target image) to the energy functional.

%For any pairs of images and metrics $(g_0,I_0)$ and 
%$(g_1,I_1)$ this leads to the Energy functional: 
%\begin{equation}
%E(\varphi)=\underset{\varphi \in \Diff(M)}{\operatorname{inf}} \operatorname{dist}_{\Diff}^{2}(\operatorname{id},\varphi)+\lambda \operatorname{dist}_{\Met}^{2}(g_0,\varphi^*g_1)+ \gamma ||I_0-I_1\circ\varphi^{-1}||^{2},\label{eq:energy}
%\end{equation}
%where $\lambda\geq 0$ and $\gamma\geq 0$ are fixed weight parameters. 

\algnewcommand{\IIf}[1]{\State\algorithmicif\ #1\ \algorithmicthen}
\algnewcommand{\EndIIf}{\unskip\ \algorithmicend\ \algorithmicif}
\begin{algorithm}[h]
\caption{Inexact Metric Matching Algorithm}\label{algo1} %using PyTorch
    \begin{algorithmic}
        \Inputs{source and target metric $g_0$, $g_1$}
        \Initialize{learning rate $\epsilon$; weight parameter $\lambda$; max iteration times $\operatorname{MaxIter}$} % shrinking factor $\epsilon$;
        \State{$\varphi,E\leftarrow\operatorname{id},0$}
        % \State{$E_{new}, E_{old}\leftarrow\infty,0$}
        \For{$\operatorname{iteration}=0:\operatorname{MaxIter}$}
            \State{$\varphi^*g_1\leftarrow(d\varphi)^T(g_1\circ\varphi)(d\varphi)$}\Comment{Pullback of $\varphi$}
            % \State{$E_{old}\leftarrow E_{new}$}\Comment{Update the energy}
            \State{$E\leftarrow \operatorname{EbinEnergy}(\varphi^*g_1,g_0,\lambda)$}\Comment{Calculate energy by Equation \eqref{eq:energy}}
            % \IIf{$ E_{new}> E_{old}$}$\alpha\leftarrow\alpha/\epsilon$\EndIIf
            \State{$v\leftarrow-\Delta^{-1}(\operatorname{E.grad})$}\Comment{Transfer gradient w.r.t. information metric to $L^2$}
            \State{$\psi\leftarrow\operatorname{id}+\epsilon v$}\Comment{Construct the approximation} %$\psi,\psi^{-1}$ to $\exp{(\alpha v)},\exp{(-\alpha v)}$
            \State{$\varphi\leftarrow\psi\circ\varphi$}\Comment{Update the diffeomorphism}
        \EndFor
        \State\Return{$\varphi$}
    \end{algorithmic}  
\end{algorithm}

\subsection{Estimating the Atlas for a Population of Connectomes.}\label{sec:atlas}
Given a collection of connectomes modeled as points on an abstract Riemannian manifold, we can directly apply least squared estimation to define the average connectome. Thus the template estimation problem can be formulated as a joint minimization problem:
\begin{align}
   \hat{g} = \underset{g,\varphi_i }{\operatorname{\rm argmin}}
    \sum_{i=1}^{N}\operatorname{dist}_{\Diff}^{2}(\operatorname{id},\varphi_i)+ \lambda \operatorname{dist}_{\Met}^{2}(g,\varphi_i^*g_i) \label{eq:atlas}
\end{align}
We use the iterative alternating algorithm proposed in \cite{joshi2004unbiased} for solving the above optimization problem: we alternate gradient steps between optimizing with respect to each diffeomorphism, $ \varphi^{-1}_i , i=1,\cdots,N $, and minimizing with respect to the metric average $\hat g$. In the metric optimization step we use the  Fr\'echet mean algorithm described in Section~\ref{sec:riemet}. See Algorithm~\ref{algo2} for details of this process.

\begin{algorithm}[h]
\caption{Atlas Building Algorithm}\label{algo2}
    \begin{algorithmic}
        \Inputs{sample metric fields list $G$}
        \Initialize{max iteration times $\operatorname{MaxIter}$}
        % \Outputs{Mean of input metric fields list $g_{\operatorname{mean}}$}
        \For{$\operatorname{iteration}=0:\operatorname{MaxIter}$}
            \State $g_{\operatorname{mean}}\leftarrow\operatorname{FrechetMean}(G)$\Comment{Section~\ref{sec:riemet}}
            \For{$i=0:\operatorname{len}(G)$}
                \State{$\varphi\leftarrow\operatorname{MetricMatching}(g_{\operatorname{mean}},G[i])$}\Comment{Algorithm~\ref{algo1}}
                \State{$G[i]\leftarrow\varphi^*G[i]$}\Comment{Update $G[i]$ by pullback of $\varphi$}
            \EndFor
        \EndFor
        \State\Return{$g_{\operatorname{mean}}$}
    \end{algorithmic}  
\end{algorithm}  

\subsection{Implementation Details}
As done in \cite{hao2014improved}, we apply a mask to both the connectome metric estimation process and the atlas building algorithm for two reasons.  First, it is important that we constrain the problem to biologically realistic white matter tracts by not allowing tractography to flow through regions of CSF.  Second, we avoid numeric issues associated with processing air and other noisy regions outside the skull.  This also speeds up computation, as we only need to look at voxels inside the masked region instead of the entire image volume.  For the atlas building algorithm, we deform each individual mask into atlas space at each outer iteration, and apply the union of these deformed masks when computing the current atlas estimate.  For each iteration of the atlas building algorithm, we perform only 2 iterations inside the metric matching function to avoid overfitting the individual metrics to early estimates of the Fr\'echet mean. In practice, we find the algorithm behaves well when we update $\epsilon$ in Algorithm~\ref{algo1} such that $1/\epsilon$ is approximately equal to the energy \eqref{eq:energy}. %At first, when the energy is relatively large, the learning rate will be relatively small, making the gradient descent slower but steady away from the solution. Correspondingly, when the energy is small, the learning rate will be relatively large, letting the gradient descent speed up as it approaches the solution.

\section{Results}
% {\bf Simulated Data:}
\paragraph{Simulated Data:}
We verified our method by generating vector fields whose central integral curves are a family of parameterized cubic functions. We used the method of parallel curves to add vectors for additional integral curves parallel to the central curve with a distance $k \in [-0.2, 0.2]$ from the central curve.  We then constructed tensors whose principal eigenvectors align with the generated vector fields and that have a specified major axis to minor axis ratio of 6:1.  

We first estimated the adaptive metric conformal to the inverse-tensor metric such that the geodesics of the adaptive metrics align with the integral curves of the simulated vector fields. After finding the connectome metric for each subject, we  ran 400 iterations of the atlas building Algorithm~\ref{algo2} to estimate the atlas in Figure \ref{fig:cubicmeans}. To help the diffeomorphisms update smoothly, we set $\lambda = 100$ in Equation \eqref{eq:energy} and the learning rate $\epsilon=5$ in Algorithm~\ref{algo1}. 

We compared a geodesic of the atlas starting from a particular seed point with geodesics of the 4 connectome metrics starting from the atlas seed point mapped into individual space. Figure \ref{fig:cubicmeans} shows these individual geodesics in atlas space before and after applying the diffeomorphisms.  We see that the atlas geodesic is nicely centered in the middle of the undeformed individual geodesics as expected.  Also, the deformed individual geodesics align well with the atlas geodesic.

\begin{figure}[h]
\centering
\includegraphics[width=\linewidth]{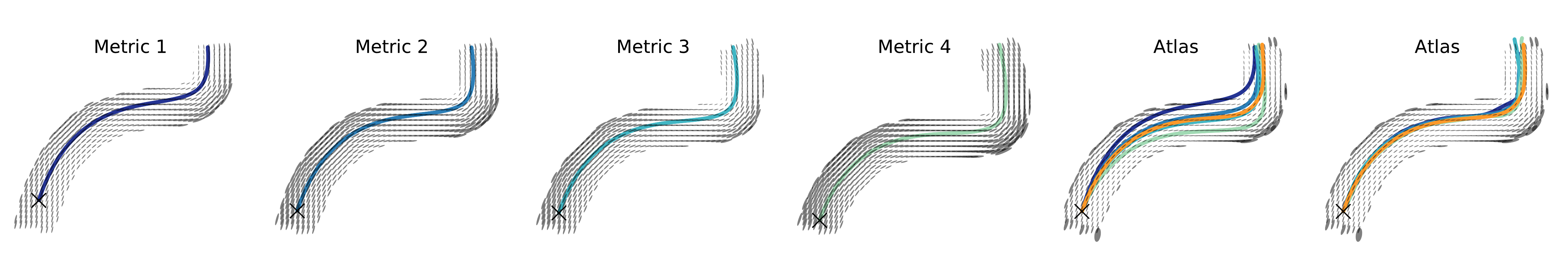}
\caption{Left: geodesics of 4 synthetic metrics starting from the atlas seed point (X) mapped into each metric's space. Second from right: estimated atlas with geodesic (orange) starting from the seed point (X) overlaid on non-deformed geodesics from each of the 4 metrics. Right: estimated atlas with geodesic (orange) overlaid on geodesics from the 4 metrics deformed into atlas space. }
\label{fig:cubicmeans}
\end{figure}
% {\bf Real Data:}
\paragraph{Real Data:}
We used a subset of subjects from the Human Connectome Project Young Adult (HCP) dataset \cite{glasser2013minimal}.  For each subject, we fit a diffusion tensor model to the images with a $b$-value of 1000 using \verb!dtifit! from FSL \cite{basser1994estimation} and generated a white-matter mask based on fractional isotropy values. We estimated the adaptive connectome metric from the inverse-tensor metric associated with the diffusion tensors.

To generate the atlas shown in Figure \ref{fig:ebinbraintoatlas}, we ran atlas building for 5000 iterations with $\lambda=100$, $\epsilon=1$, which took 50 minutes on an Intel Xeon Silver 4108 CPU. The regularization term, $\lambda$, balances the magnitudes of the diffeomorphisms from each subject's connectome metric to the atlas. To ensure that the final geodesics in the atlas also follow the major eigenvectors of the atlas tensors, we solve for the $\alpha$ conformal factor for the atlas as described in Section~\ref{sec:structasmet}. 

\begin{figure}[h]
\centering
\includegraphics[width=\linewidth]{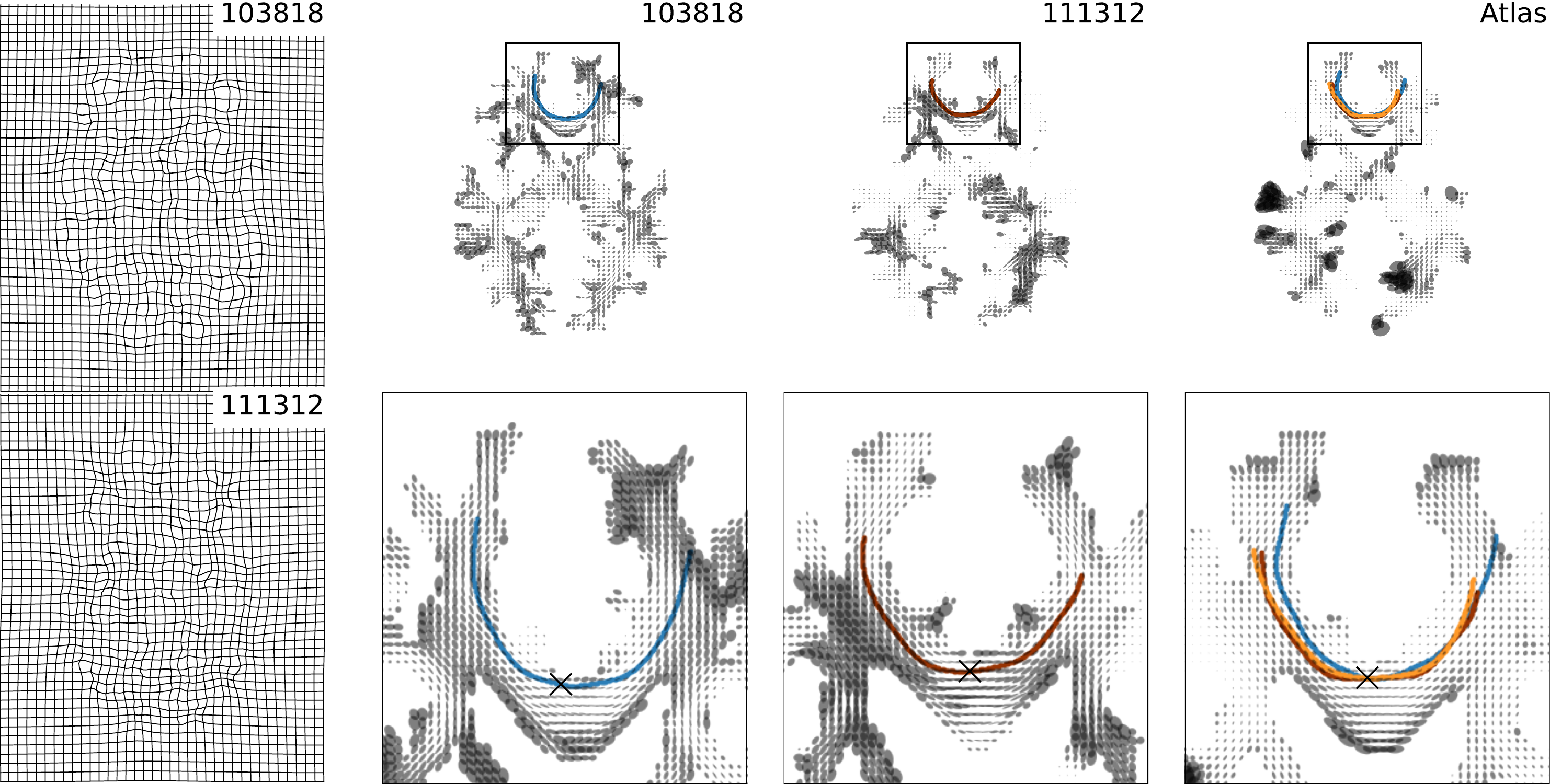}
\caption{Left: diffeomorphism from HCP subjects (103818, 111312) to the atlas. Center: each subject's connectome metric and a geodesic (blue, red) starting from the atlas seed (X) mapped to subject space.  Right: atlas and a geodesic (orange) starting at the seed (X). Subject geodesics are mapped to atlas space (blue, red). Bottom: detailed view of corpus callosum.}
\label{fig:ebinbraintoatlas}
\end{figure}

% Caption shortened to fit paper to 12 pages
%\caption{Geodesics in individual brain space mapped to atlas space. Left column: diffeomorphism from each HCP subject (103818, 111312) to the atlas. Center 2 columns: each subject's conformally scaled diffusion tensors.  Blue and red curves correspond to a geodesic for each subject starting from the atlas seed point (X) mapped to individual space.  Right column: atlas and an atlas geodesic (orange) starting at the seed (X). Subject geodesics are mapped into atlas space (blue, red). Bottom: detailed view of the corpus callosum.}
% 

\section{Conclusions}
In this paper, we introduce a novel framework for statistically analyzing structural connectomes by representing them as a point on the manifold of Riemannian metrics, enabling us to perform geometric statistics. Using this representation, we build a framework for connectome atlas construction based on the action of the diffeomorphism group and the natural Ebin metric on the space of all Riemannian metrics. Although the Ebin metric is canonical, it is not the only diffeomorphism-invariant metric available on the space of all Riemannian metrics, c.f.~\cite{bauer2013sobolev}.  Our framework allows for other choices of metrics and regularization terms, which we will explore more fully in future work. We also plan to investigate in more detail the convergence properties of the proposed algorithms, the impact of the parameter choice on results, and comparisons to other existing methods. We expect this new methodology to open up opportunities for a deeper understanding of structural connectomes and their variabilities.

\bibliographystyle{splncs04}
\bibliography{StructuralConnectomeAtlas}

\end{document}

%% file: MathNotation.tex
\def\ka{\kappa}

\def\Ga{\Gamma}

\def\ver{\on{\ver}}

\def\Id{\on{Id}}
\def\Diff{\on{Diff}}

\def\x{\times}

\def\R{{\mathbb R}}

\def\exp{\operatorname{exp}}

\let\on=\operatorname

\def\Tr{\on{Tr}}
\def\vol{\on{vol}}
\def\dist{\on{dist}}
\def\Met{\on{Met}}